\documentclass[runningheads]{llncs}
\usepackage[T1]{fontenc}
\usepackage{bbding}
\usepackage{graphicx}
\usepackage{xcolor}
\usepackage{hyperref}
\usepackage{url}

\usepackage{algorithm}
\usepackage{algorithmicx} 
\usepackage{algpseudocode}

\usepackage{amsmath,amsfonts,bm}

\def\eqref#1{equation~\ref{#1}}

\def\1{\bm{1}}

\DeclareMathAlphabet{\mathsfit}{\encodingdefault}{\sfdefault}{m}{sl}
\SetMathAlphabet{\mathsfit}{bold}{\encodingdefault}{\sfdefault}{bx}{n}

\newcommand{\Kset}{\mathcal{K}}
\newcommand{\Fset}{\mathcal{F}}
\newcommand{\duoE}{\mathbb{E}}
\newcommand{\duoP}{\mathbb{P}}
\newcommand{\Real}{\mathbb{R}}
\begin{document}
\title{Adversarial bandit optimization for approximately linear functions}
%
%

\author{Zhuoyu Cheng \Envelope\inst{1}\orcidID{0009-0002-6631-4929} \and
Kohei Hatano\inst{2,3}\orcidID{0000-0002-1536-1269} \and
Eiji Takimoto\inst{2}\orcidID{0000-0001-9542-2553}}
\authorrunning{Z. Cheng et al.}
%
\institute{Joint Graduate School of Mathematics for Innovation, Kyushu University, Japan
\and
Department of Informatics, Kyushu University, Japan \and
RIKEN AIP, Japan \\
\noindent\Envelope\email{cheng.zhuoyu.874@s.kyushu-u.ac.jp} \\
\email{\{hatano,eiji\}@inf.kyushu-u.ac.jp}}
\maketitle              
\begin{abstract}
We consider a bandit optimization problem for non-convex and non-smooth functions, where in each trial the loss function is the sum of a linear function and a small but arbitrary perturbation chosen after observing the player’s choice. We give both expected and high probability regret bounds for the problem. Our result also implies an improved high-probability regret bound for the bandit linear optimization, a special case with no perturbation. We also give a lower bound on the expected regret.

\end{abstract}
\section{Introduction}
Bandit optimization is a sequential game between a player and an adversary. 
The game is played over $T$ rounds, where $T$ is a positive natural number called the horizon. 
The game is specified by a pair $(\mathcal{K}, \mathcal{F})$, 
where $\mathcal{K}\subseteq\Real^{d}$ is a bounded closed convex set and 
$\mathcal{F}\subseteq\{ f: \mathcal{K}\to\Real\}$ is a function class.
In each round $t\in[T]$, the player first chooses an action $x_{t}\in\mathcal{K}$ and 
the adversary chooses a loss function $f_{t}\in\mathcal{F}$,
and then the player receives the value $f_{t}(x_{t})$ as the loss.
Note that $f_{t}$ itself is unknown to the player.
In this paper, we assume the adversary is oblivious, i.e., the loss functions are specified before starting the game
\footnote{We do not consider the case where the adversary is adaptive, i.e., 
it can choose the $t$-th loss function $f_t$ depending on the previous actions $x_1,\dots,x_{t-1}$.}.
The goal of the player is to minimize the regret 
\begin{align}
\label{regret}
\sum_{t=1}^{T}f_{t}(x_{t})-\min_{x\in\mathcal{K}}\sum_{t=1}^{T}f_{t}(x) 
\end{align}
in expectation (expected regret) or with high probability (high-probability regret).

For convex loss functions,
the bandit optimization has been extensively studied (see, e.g.,\cite{abernethy2012interior,dani2007price,lee2020bias}). 
 $\mathcal{O}(d^{1/3}T^{3/4})$ regret bounds are shown by Flaxman et al.~\cite{flaxman-etal:soda05}.
Lattimore~\cite{lattimore:msl20} shows an information-theoretic regret bound $\widetilde{O}(d^{2.5}\sqrt{T})$ 
for convex loss functions.
For linear loss functions, Abernethy et al.~\cite{abernethy2012interior} propose the SCRiBLe algorithm  and
give an expected regret bound $\mathcal{O}(d\sqrt{T\ln T})$, 
achieving optimal dependence on $T$~\cite{bubeck2012towards}.
Lee et al.~\cite{lee2020bias} propose SCRiBLe with lifting and increasing learning rates and show a high-probability regret bound
$\widetilde{O}(d^2\sqrt{T})$.

Recently, 
non-convex functions are also getting popular in this literature.
For example, Agarwal et al.~\cite{agarwal2019learning} show a regret bound $\mathcal{O}(poly(d)T^{2/3})$ 
for smooth and bounded non-convex functions.
Ghai et al.~\cite{ghai2022non} propose algorithms with regret bounds $\mathcal{O}(poly(d)T^{2/3})$ 
under the assumption that non-convex functions are reparameterized as some convex functions.

In this paper, we investigate the bandit optimization problem for a class of non-convex and non-smooth loss functions. 
The function class consists of non-smooth and non-convex functions that are "close" to linear functions, 
in the sense that 
functions in the class can be viewed as 
linear functions with adversarial non-convex perturbations whose amount is up to $\epsilon$. Bandit optimization for linear loss functions with stochastic noise (e.g., \cite{abbasi2011improved,amani2019linear}) 
cannot be applied to our problem.
We propose a novel approach to analyzing high-probability regret, introducing a new method for decomposing regret.

\begin{enumerate}
    \item When $\epsilon \neq 0$, we propose a modification of SCRiBLe with lifting and increasing learning rates~\cite{lee2020bias} and leverage the properties of the $\nu$-normal barrier~\cite{nemirovski2004interior} to prove its high probability regret bound $\widetilde{O}(d\sqrt{T}+\sqrt{\epsilon }dT+\epsilon dT)$. We also derive its expected regret $\widetilde{O}(d\sqrt{T}+\sqrt{\epsilon }dT+\epsilon dT)$.
    \item When $\epsilon=0$, this problem becomes bandit linear optimization, a special case with no perturbation. Compared to the result of~\cite{lee2020bias}, holding with probability $1-\gamma$, $\mathcal{O}(\ln^{2}(dT)d^{2}\ln T \sqrt{T\ln \frac{\ln (dT)}{\gamma}})$, we use a different regret decomposition approach to achieve a better high-probability regret bound $\mathcal{O}(d\sqrt{T\ln T}+\ln T\sqrt{T\ln(\frac{\ln T}{\gamma})}+\ln(\frac{\ln T}{\gamma}))$.
    \item  We prove a lower bound $\Omega(\epsilon T)$, characterizing the minimal dependence on the parameter $\epsilon$.
    
\end{enumerate}


\section{Related Work}
Bandit linear optimization was first proposed by  Awerbuch \& Kleinberg~\cite{awerbuch2004adaptive}, 
who achieved a regret bound of $\mathcal{O}(d^{3/5}T^{2/3})$ against oblivious adversaries.
Later, McMahan \& Blum~\cite{mcmahan2004online} established a regret bound of $\mathcal{O}(dT^{3/4})$ when facing adaptive adversaries. 
A foundational approach in bandit optimization problems involves gradient-based smoothing techniques. 
Flaxman et al.~\cite{flaxman-etal:soda05} proposed a method for constructing an unbiased estimator of the loss function's gradient in the bandit setting.
Abernethy et al.~\cite{abernethy2012interior} introduced the SCRiBLe algorithm and achieved an expected regret bound of $\mathcal{O}(d\sqrt{T\ln T})$ against oblivious adversaries.
Bartlett et al.~\cite{bartlett2008high} proposed a high-probability regret bound of $\mathcal{O}(d^{2/3}\sqrt{T\ln dT})$ under a special condition. Subsequently, Lee et al.~\cite{lee2020bias} presented a high-probability regret bound $\widetilde{O}(d^2\sqrt{T})$ for both oblivious and adaptive adversaries. 
In recent work, 
Ito and Takemura~\cite{ito2023best,ito2023exploration} proposed a  bandit linear algorithm that adapts to both stochastic and adversarial environments, achieving a regret bound of 
$\mathcal{O}(d\sqrt{T\ln T})$ in the adversarial setting.
Rodemann et al.~\cite{rodemann2024reciprocal} established a connection between bandit problems and Bayesian black-box optimization, offering theoretical foundations for regret bounds across both domains.

Unlike bandit convex optimization problems, which have been extensively explored and analyzed, bandit non-convex optimization problems introduce unique challenges due to the complexity of exploring and exploiting in a non-convex area. Gao
et al.~\cite{gao2018online} considered both non-convex losses and non-stationary data and established a regret bound of $O(\sqrt{T+poly(T)}$. Yang et al.~\cite{yang2018optimal} achieved a regret bound of $O(\sqrt{T\ln T})$ for non-convex loss functions. However, they both require the loss functions to have smoothness properties, and our loss functions are neither convex nor smooth.




\subsection{Comparison to  Lee et al.~\cite{lee2020bias}}

Our approach builds upon  Lee et al.'s work. Below, we highlight the key differences between our method and  Lee et al.'s in the context of the oblivious bandit setting:

\begin{enumerate}
\item Simplified Regret Analysis:
While Lee et al.'s regret analysis adds unnecessary complexity in the oblivious bandit setting, our approach offers a simplified analysis with clearer and more interpretable results.

\item Reduced Dependence on $d$:
 Lee et al.'s analysis results in a regret bound with greater dependence on $d$, whereas our method derives a bound with significantly reduced dependence on $d$ (This distinction is demonstrated in the introduction and further illustrated in the case where $\epsilon = 0$).

\item Revised Generality of Problem Setting: 
Like the SCRiBLe algorithm~\cite{abernethy2012interior}, our approach is more general, treating bandit linear optimization as a special case within a broader problem framework.

\end{enumerate}



\section{Preliminaries}
 This section introduces some necessary notations and defines an $\epsilon$-approximately linear function. Then we give our problem setting.

\subsection{Notation}
We abbreviate the $2$-norm $\| \cdot \|_2$ as $\|\cdot\|$. 
For a twice differentiable convex function $\mathcal{R}: \Real^d \to \Real$ and 
any $x, h\in \Real^d$, let  
$\| h \|_{x}=
\| h \|_{\nabla^{2}\mathcal{R}(x)}=\sqrt{h^{\top}\nabla^{2}\mathcal{R}(x)h}$, 
and 
$\| h \|_{x}^{*}=
\| h \|_{(\nabla^{2}\mathcal{R}(x))^{-1}}=\sqrt{h^{\top}(\nabla^{2}\mathcal{R}(x))^{-1}h}$, 
respectively.

For any $v\in \Real^d$, 
let $v^{\perp}$ be the space orthogonal to $v$.
Let $\mathbb{S}^d_{1}=\{x\mid\| x \| = 1\}$.
The vector $e_{i} \in \Real^d$ is a standard basis vector with a value of $1$ in the 
$i$-th position and $0$ in all other positions.
$I$ is an identity matrix with dimensionality implied by context.

\subsection{Problem Setting}
Let $\mathcal{K}\subseteq\Real^{d}$ be a bounded and closed ball of radius $D$ centered at the zero vector.
Furthermore, we assume that $\Kset$ contains the unit ball. 
Otherwise, 
we can apply an affine transformation to translate the center point of the convex set to the origin. 
Let $\mathcal{K'}=\{(x,1):x\in\mathcal{K}\}$.
For any $\delta \in (0,1)$, let 
$\mathcal{K}_{\delta}=\{x|\frac{1}{1-\delta} x \in \mathcal{K}\}$ and 
$\mathcal{K'_{\delta}}=\{(x,1):x\in\mathcal{K_\delta}\}$, respectively.

\begin{definition} 
A function $f: \mathcal{K} \to \Real$ is 
$\epsilon$-approximately linear
if there exists $\theta_f \in \Real^d$ such that 
$\forall x\in\mathcal{K}$, $|f(x) - \theta_f^\top x| \leq \epsilon$. 
\end{definition}
For convenience, in the definition above, 
let $\sigma_f(x)=f(x)-\theta_f^\top x$, and we omit the subscript $f$ of $\theta_f$ and $\sigma_f$ 
if the context is clear.
Note that $|\sigma(x)|\leq\epsilon$ for any $x \in \mathcal{K}$.  

In this paper, we consider the bandit optimization $(\Kset, \Fset)$, where 
$\Fset$ is the set of $\epsilon$-approximately linear functions $f(x)=\theta^\top x +\sigma(x)$ 
with $\|\theta\| \leq G$ 
and we assume that
$
|f(x)| \leq 1, \forall x \in \mathcal{K}.$
Bandit optimization for $\epsilon$-approximately linear functions can be defined as the following statement.
For every round $t=1,..,T$, the player selects an action 
$x_{t}\in \mathcal{K}$, without knowing the loss function in advance. The environment, modeled as an oblivious adversary, chooses a sequence of linear loss vectors 
$\theta_1, \theta_2,...,\theta_t$
 before the interaction begins and independent of the player's actions. After selecting 
$x_{t}$, the adversary chooses a perturbation $\sigma_{t}(x_{t})(|\sigma_{t}(x_{t})|\leq\epsilon)$
\footnote{Since we assume that $|f(x)| \leq 1, \forall x \in \mathcal{K}$,  we only consider the case where $\epsilon < 1$.}
and the player observes only the scalar loss 
$f_{t}(x_{t})(=\theta^\top_{t} x_{t}+\sigma_{t}(x_{t}))$.
The goal of the player is to minimize the regret $\sum_{t=1}^{T}f_{t}(x_{t})-\min_{x\in\mathcal{K}}\sum_{t=1}^{T}f_{t}(x)$.

\section{Main Results}
In this section, we first introduce SCRiBLe with lifting, followed by presenting the main contributions of this paper with detailed explanations.

\begin{algorithm}[h]
    \caption{SCRiBLe with lifting}
    \label{alg3}
    \begin{algorithmic}[1]
        \Require
        
        $T$, parameters $\eta\in\Real, \delta\in(0, 1)$,
        $\nu$-normal barrier $\mathcal{R}$ on $con(\mathcal{K})$
        \State Initialize: $x'_1=\arg\min_{x'\in\mathcal{K'_{\delta}}}{\mathcal{R}(x')}$
        \For{$t=1,..,T$}
            \State let $\mathbf{A}_{t}=[\nabla^2\mathcal{R}(x'_{t})]^{-\frac{1}{2}}$ \State  Draw $\mu_{t}$ from $\mathbb{S}_{1}^{d+1}\cap(\mathbf{A}_{t}e_{d+1})^{\perp}$ uniformly, set $y'_t=(y_{t},1)=x'_{t}+\mathbf{A}_{t}\mu_{t}$.  
            \State Play $y_{t}$, observe and incur loss $f_{t}(y_{t})$. Let $g_{t}=df_{t}(y_{t})\mathbf{A}_{t}^{-1}\mu_{t}$.
            \State Update $x'_{t+1}=\arg\min\limits_{x'\in\mathcal{K'_{\delta}}}{\eta\sum_{\tau=1}^{t}g_{\tau}^{\top}x'+\mathcal{R}(x')}$
        \EndFor
    \end{algorithmic}
\end{algorithm}

\subsection{SCRiBLe with lifting}
For the decision set $\mathcal{K}$ with a $\nu$-normal barrier on $con(\mathcal{K})$, where $con(\mathcal{K})=\{\mathbf{0}\}\cup\{(x,b):\frac{x}{b}\in\mathcal{K}, x\in\Real^{d}, b > 0\}$, we apply Algorithm 1 to $\epsilon$-approximately linear functions. 
Recall $\mathcal{K'_{\delta}}=\{(x,1):x\in\mathcal{K_\delta}\}$.

We simplify SCRiBLe with lifting and increasing learning rates~\cite{lee2020bias}.
The algorithm performs original SCRiBLe~\cite{abernethy2012interior} over a lifted decision set, using a $\nu$-normal barrier $\mathcal{R}$ defined over the $con(\mathcal{K})$ (which always exists) as the regularizer to generate the sequence $x'_{1},...,x'_{t}$. 
It set $y'_{t}=x'_{t}+\mathbf{A}_{t}\mu_{t}$, where $\mathbf{A}_{t}=[\nabla^2\mathcal{R}(x'_{t})]^{-\frac{1}{2}}$ and $\mu_{t}$ is uniformly sampled at random from $\mathbb{S}_{1}^{d+1}\cap(\mathbf{A}_{t}e_{d+1})^{\perp}$. Since $\mu_{t}$ is orthogonal to
$\mathbf{A}_{t}e_{d+1}$, the last coordinate of $\mathbf{A}_{t}\mu_{t}$ is zero, ensuring that $y'_{t}=(y_{t}, 1)$ remains within $\mathcal{K'}$. 
The actual point played is still $y_{t}$.
After playing $y_{t}$ and observing $f_{t}(y_{t})(=\theta_{t}^{\top}y_{t}+\sigma_t(y_{t}))$, it constructs the loss estimator the same way as SCRiBLe algorithm~\cite{abernethy2012interior}: $g_{t}=df_{t}(y_{t})\mathbf{A}_{t}^{-1}\mu_{t}$. Furthermore, it adopts the same update method as FTRL algorithm~\cite{hazan2016introduction}: $x'_{t+1}=\arg\min\limits_{x'\in\mathcal{K'}}{\eta\sum_{\tau=1}^{t}g_{\tau}^{\top}x'+\mathcal{R}(x')}$.

Compared with SCRiBLe with lifting and increasing learning rates~\cite{lee2020bias}, which focuses on bandit linear optimization, we do not use the increasing learning rates part, but retain the lifting.
This preserves its advantages; for instance, 
the $\nu$-normal barrier $\mathcal{R}$ always exists on $con(\mathcal{K})$ and we can leverage the properties of the 
$\nu$-normal barrier.
When the loss takes the form $f_{t}(y_{t})=\theta_{t}^{\top}y_{t}$ (i.e., when $\epsilon=0$), the analysis by Lee et al.~\cite{lee2020bias} shows that the first $d$ coordinates of $g_{t}$ are indeed an unbiased estimator of $\theta_{t}$.
However, when $\epsilon\ne0$ since $g_{t}$ will always be influenced by $\sigma_{t}(y_t)$, $g_{t}$ is no longer an unbiased estimator of $\theta_{t}$.



We present our main results: expected and high-probability regret bounds for the problem.
\begin{theorem}
\label{Theo:expected-regret}
The algorithm with parameters
\(
\eta = \frac{\sqrt{\nu\ln \frac{1}{\delta}}}{2d\sqrt{T}}
\) and
\[
\delta =
\begin{cases}
\frac{1}{T^{2}}, & \epsilon = 0, \\
\sqrt{\epsilon}, & \epsilon \neq 0,
\end{cases}
\]
guarantees the following expected regret bound
\begin{equation}
    \mathbb{E}\Bigg[
    \sum_{t=1}^{T} f_t(y_t)
    - \min_{x\in \mathcal{K}} \sum_{t=1}^{T} f_t(x)
    \Bigg]
    \leq
    4d\sqrt{\nu T\ln \frac{1}{\delta}}
    + 2dT(\nu +2\sqrt{\nu })(\frac{1-\delta}{\delta})\epsilon + \delta GDT + 2T\epsilon.
\end{equation}
\end{theorem}

\begin{theorem}
\label{Theo:high-probability-regret}
The algorithm with parameters $\eta=\frac{\sqrt{\nu \ln \frac{1}{\delta}}}{2d\sqrt{T}}$
and
\[
\delta =
\begin{cases}
\frac{1}{T^{2}}, & \epsilon = 0, \\
\sqrt{\epsilon}, & \epsilon \neq 0,
\end{cases}
\]
ensures that with probability at least $1-\gamma$
    \[
    \begin{array}{l}
        \sum_{t=1}^{T}f_t(y_{t})-\min_{x\in\mathcal{K}}\sum_{t=1}^{T}f_t(x)
                \leq 4d\sqrt{\nu T\ln \frac{1}{\delta}}
        + 2dT(\nu +2\sqrt{\nu })(\frac{1-\delta}{\delta})\epsilon + \delta GDT \\
        +2T\epsilon + C(2GD\ln\frac{C}{\gamma}
        +(1+\epsilon)\sqrt{8T \ln\frac{C}{\gamma}})
    \end{array}
    \]
    where $C=\lceil \ln  GD \rceil\lceil \ln ((GD)^{2}T) \rceil$.
\end{theorem}

\subsection{Analysis}
We primarily divide the regret into parts:
\begin{equation}
    Regret = \underbrace{
\sum_{t=1}^{T}(x'_{t}-h)^{\top}E[g_{t}]
}_{\textsc{Reg-Term}}+\underbrace{
\sum_{t=1}^{T}(y'_{t}-x'_{t})^{\top}\theta'_{t}
}_{\textsc{Deviation-Term}}+\underbrace{\sum_{t=1}^{T}
(d\sigma_{t}(y_{t})\mathbf{A}_{t}^{-1}\mu_{t})^{\top}(x'_{t}-h)
}_{\textsc{Error-Term}},
\end{equation}
where $h\in\mathcal{K'}$, $\theta'=(\theta,0)$.
This differs from the decomposition proposed by Lee et al.~\cite{lee2020bias}, which divides the regret as follows:
\begin{equation}
    Regret = \underbrace{\sum_{t=1}^{T}(x'_{t}-h)^{\top}g_{t}}_{\textsc{Reg-Term}}+\underbrace{
\sum_{t=1}^{T}[{y'_{t}}^{\top}{\theta'_{t}}-{x'_{t}}^{\top}g_{t}+h^{\top}(g_{t}-{\theta'_{t}})]
}_{\textsc{Deviation-Term}},
\end{equation}
where $h\in\mathcal{K'}$, $\theta'=(\theta,0)$.
This means that, for an oblivious adversary, calculating the regret does not require considering the variance of the estimator $g_{t}$ and $\theta_{t}$, but only the variance between $y_{t}$ and $x_{t}$. This difference is a key factor that enables us to achieve a better high-probability regret bound when $\epsilon=0$. In addition, we need to account for an extra \textsc{Error-Term} when 
$\epsilon\ne0$.

The bound on the \textsc{Reg-Term} is identical to that in previous work~\cite{lee2020bias}, so we omit its explanation and first focus on the \textsc{Error-Term}. We make the entire analysis hold in $\mathbb{R}^{d+1}$.
The Cauchy-Schwarz inequality helps in deriving the bounds \textsc{Error-Term}$\leq \| d\sigma_{t}(y_{t})\mathbf{A}_{t}^{-1}\mu_{t}
\|_{\nabla^2\mathcal{R}(x'_{t})}^{*}\| x'_{t}-h \|_{\nabla^2\mathcal{R}(x'_{t})}$ and $\| d\sigma_{t}(y_{t})\mathbf{A}_{t}^{-1}\mu_{t} \|_{\nabla^2\mathcal{R}(x'_{t})}^{*}\leq d\epsilon$.
However, since the largest eigenvalue of $\nabla^2\mathcal{R}(x'_{t})$~\cite{nemirovski2004interior} could potentially approach infinity, the main challenge in bounding the \textsc{Error-Term} lies in the difficulty of bounding $\| x'_{t}- h\|_{\nabla^2\mathcal{R}(x'_{t})}$, has not been addressed in previous work.
Lee et al.~\cite{lee2020bias} derive a related inequality:
$\| h\|_{\nabla^2\mathcal{R}(x'_{t})}\leq -\| h\|_{\nabla^2\mathcal{R}(x'_{t+1})}+\nu\ln (\nu T+1)$, but it is still too large and does not help in bounding $\| x'_{t}- h\|_{\nabla^2\mathcal{R}(x'_{t})}$.
We present a straightforward yet necessary Lemma~\ref{lemm:control-h}, which helps to bound $\| x'_{t}-h\|_{\nabla^2\mathcal{R}(x'_{t})}\leq 2(\frac{1}{\delta}-1)(\nu+2\sqrt{\nu})$, where $0<\delta<1$. 
This also implies that increasing learning rates are not required, as they are solely aimed at bounding 
$\| h\|_{\nabla^2\mathcal{R}(x'_{t})}$ in Lee et al.'s paper~\cite{lee2020bias}.

Secondly, by obtaining the expected bound and high-probability bound for \textsc{Deviation-Term}, we can derive the expected regret bound and high-probability regret bound, respectively. 
For the high-probability bound of the \textsc{Deviation-Term} when 
$\epsilon=0$, our approach differs from SCRiBLe with lifting and increasing learning rates~\cite{lee2020bias}, which constrains $\delta$ to satisfy certain conditions to ensure that $x'_{t}$ is
never too close to the boundary (thus ensuring that the eigenvalues of $\mathbf{A}_{t}$ are bounded, especially for bounding $\| h\|_{\nabla^2\mathcal{R}(x'_{t})}$).
Our approach does not require $x'_{t}$ to stay away from the boundary. 
Furthermore, we do not need to bound the eigenvalues of $\mathbf{A}_{t}$, which gives us greater flexibility in choosing the value of $\delta$ (such as $\frac{1}{T^{2}}$), leading to a better upper bound for the regret.

Finally, we prove the lower bound of regret in section 6.

\section{Proof}
This section introduces preliminaries on the 
$\nu$-normal barrier, presents several essential lemmas and provides proofs of the main theorems.

\subsection{$\nu$-normal barrier}
We introduce the $\nu$-normal barrier, providing its definitions and highlighting several key properties that will be frequently used in the subsequent analysis.

\begin{definition}
     Let $\Psi\subseteq\Real^{d}$ be a closed and proper convex cone and let $\nu\geq 1$. A function $\mathcal{R}:int(\Psi)\to \Real$: is called a $\nu$-logarithmically homogeneous self-concordant barrier (or simply $\nu$-normal barrier) on $\Psi$ if 
    \begin{enumerate}
        \item $\mathcal{R}$ is three times continuously differentiable and convex and approaches infinity along any sequence of points approaching the boundary of $\Psi$. 
        \item For every $h\in\Real^{d}$ and $x\in int(\Psi)$ the following holds: 
    \begin{equation}
        \sum_{i=1}^{d}\sum_{j=1}^{d}\sum_{k=1}^{d}\frac{\partial^{3}\mathcal{R}(x)}{\partial x_{i}\partial x_{j}\partial x_{k}}h_{i}h_{j}h_{k}\leq 2 \| h \|^{3}_{x},
    \end{equation}
    \begin{equation}
        \|\nabla \mathcal{R}(x)^{\top}h \|\leq \sqrt{\nu}\| h \|_{x},
    \end{equation}
    \begin{equation}
        \mathcal{R}(tx)=\mathcal{R}(x)-\nu \ln t, \forall x\in int(\Psi), t>0.
    \end{equation}
    \end{enumerate}
\end{definition}

\begin{lemma}[\cite{nemirovski2004interior,nesterov1994interior}]
\label{lemm:base-normal-barrier}
If $\mathcal{R}$ is a $\nu$-normal barrier on $\Psi$, Then for any $x\in int(\Psi)$ and any $h\in\Psi$, we have
    \begin{equation}
    \label{eq:lemma1}
        \| x \|_{x}^{2} = \nu,
    \end{equation}
    \begin{equation}
        \nabla^{2}\mathcal{R}(x)x=-\nabla\mathcal{R}(x),
    \end{equation}
    \begin{equation}
        \| h \|_{x} \leq -\nabla\mathcal{R}(x)^{\top}h,
    \end{equation}
    \begin{equation}
        \nabla\mathcal{R}(x)^{\top}(h-x) \leq \nu .
    \end{equation}
\end{lemma}
\begin{lemma}[~\cite{nemirovski2004interior}]
\label{lemm:normal-property}
    If $\mathcal{R}$ is a $\nu$-normal barrier on $\Psi$, then the Dikin ellipsoid centered at $x\in int(\Psi)$, defined as $\{y: \| y-x\|_{x}\leq 1\}$, is always within $\Psi$. Moreover,
 \begin{equation}
     \| h\|_{y}\geq \| h\|_{x}(1-\| y-x\|_{x})
 \end{equation}
 holds for any $h\in\Real^{d}$ and any $y$ with $\| y-x\|_{x}\leq 1$.
\end{lemma}

\begin{lemma}[~\cite{hazan2016introduction}]
\label{lemm:normal-log-property}
        Let $\mathcal{R}$ is a $\nu$-normal barrier over $\Psi$, then for all $x, z\in int(\Psi): \mathcal{R}(z)-\mathcal{R}(x)\leq \nu \ln \frac{1}{1-\pi_{x}(z)}$, where $\pi_{x}(z)=\inf\{t\geq 0 : x+t^{-1}(z-x)\in\Psi\}$.
\end{lemma}



\subsection{Useful lemmas}
In addition to the properties of the normal barrier and its related lemmas, we also need to introduce some additional necessary lemmas.

Lemma~\ref{lemm:control-h} plays a crucial role in our analysis, especially in the proof of Lemma~\ref{lemm:bound-regret}.
\begin{lemma}
\label{lemm:control-h}
    For Algorithm 1 and for any $x,y \in \mathcal{K'_{\delta}}$
    \begin{equation}
        \| x-y \|_{x}
        \leq 2(\frac{1}{\delta}-1)
        \left(\nu+2\sqrt{\nu}\right).
    \end{equation}

\end{lemma}

Like Lemma 6 in the SCRiBLe algorithm{~\cite{abernethy2008competing}, the next minimizer $x'_{t+1}$ is “close” to $x'_{t}$ and Lemma~\ref{lemm:close-property} implies $\| x'_{t+1}-x'_{t}\|_{x'_{t}}\leq 4d\eta$. This result will help us bound $g_{t}^{\top} (x'_{t}-h)$, where $h\in\mathcal{K'}$(see Lemma~\ref{lemm:FTRL}).

\begin{lemma}
\label{lemm:close-property}
        $x'_{t+1}\in W_{4d\eta}(x'_{t})$, where $W_{r}(x')=\{y\in \mathcal{K'}: \| y-x'\|_{x'}< r\}$.
\end{lemma}

This next lemma is based on Lemma B.9.~\cite{lee2020bias}, but due to the differences in the loss functions, $g_{t,i}$ is no longer an unbiased estimate of $\theta_{t,i}$, for $i\in[d]$. Lee et al. state that $\duoE_{t}[l_{t,i}]=\theta_{t,i}$, for $i\in[d]$, where $l_{t}=d(\theta_{t}, 0)(x_{t}'+\mathbf{A}_{t}\mu_{t})\mathbf{A}_{t}^{-1}\mu_{t}$~\cite{lee2020bias}. We directly apply it to our analysis.
\begin{lemma}
\label{lemm:gradient-estimate}
Let $l_{t}=d(\theta_{t}, 0)(x_{t}'+\mathbf{A}_{t}\mu_{t})\mathbf{A}_{t}^{-1}\mu_{t}$. For Algorithm 1, we have $\duoE_{t}[l_{t,i}]=\theta_{t,i}$, for $i\in[d]$.
\end{lemma}

The regret bound of FTRL algorithm~\cite{hazan2016introduction} states that for every $u\in \mathcal{K}$, $\sum_{t=1}^{T}\nabla_{t}^{\top}x_{t}-\sum_{t=1}^{T}\nabla_{t}^{\top}u \leq \sum_{t=1}^{T}[\nabla_{t}^{\top}x_{t}-\nabla_{t}^{\top}x_{t+1}]+\frac{1}{\eta}[\mathcal{R}(u)-\mathcal{R}(x_{1})]$, where $\nabla_{t}$ represents the gradient of the loss function $f_t$. In our adaptation, we replaced $\nabla_{t}$ with $g_{t}$ and $\mathcal{K}$ with 
$\mathcal{K'}$. This modification does not fundamentally alter the original result.
Since the update way $x'_{t+1}=\arg\min\limits_{x'\in\mathcal{K'}} {\eta\sum_{\tau=1}^{t}g_{\tau}^{\top}x'+\mathcal{R}(x')}$
satisfied the condition of FTRL algorithm~\cite{hazan2016introduction}, we can apply Lemma 5.3. in ~\cite{hazan2016introduction} to Algorithm 1 as follow.

\begin{lemma}
\label{lemm:FTRL}
        For Algorithm 1 and for every $h\in \mathcal{K'}$,
        $\sum_{t=1}^{T}g_{t}^{\top}x'_{t}-\sum_{t=1}^{T}g_{t}^{\top}h \leq \sum_{t=1}^{T}[g_{t}^{\top}x'_{t}-g_{t}^{\top}x'_{t+1}]+\frac{1}{\eta}[\mathcal{R}(h)-\mathcal{R}(x'_{1})]$.
\end{lemma}

    
    
    

The lemmas presented above are all intended to lead to the derivation of Lemma~\ref{lemm:bound-regret}, which serves as a key result of this paper. Specifically, it provides a bound for $\sum_{t=1}^{T}\theta_{t}^\top x_{t}-\sum_{t=1}^{T}\theta_{t}^\top x^{*}$, where $x^{*}=\arg\min_{x\in \mathcal{K}}\sum_{t=1}^{T}f_{t}(x)$. 
Due to Lemma~\ref{lemm:bound-regret}, we only need to consider $\sum_{t=1}^{T}\theta_{t}^\top y_{t}-\sum_{t=1}^{T}\theta_{t}^\top x_{t}$ when calculating the regret bound.
This result plays a crucial role in deriving both expected and high-probability regret bounds.

\begin{lemma}
\label{lemm:bound-regret}
    For Algorithm 1, let $f_{t}(x_t)=\theta_{t}^{\top}x_{t}+\sigma_{t}(x_{t})$ and $x^{*}=\arg\min_{x\in \mathcal{K}}\sum_{t=1}^{T}f_{t}(x)$ and we have
    \begin{equation}
         \sum_{t=1}^{T}\theta_{t}^\top x_{t}-\sum_{t=1}^{T}\theta_{t}^\top x^{*}
                \leq 4\eta d^{2}T+\frac{\nu \ln (\frac{1}{\delta})}{\eta}+ 2dT(\nu +2\sqrt{\nu })(\frac{1-\delta}{\delta})\epsilon + \delta DGT .
    \end{equation}
\end{lemma}

With the help of Lemma~\ref{lemm:bound-regret}, we are ready to present the proof of Theorem~\ref{Theo:expected-regret}.

\subsection{Proof of Theorem~\ref{Theo:expected-regret}}

\begin{proof} 
    Recall an $\epsilon$-approximately linear function can be written as: $f(x)=\theta^\top x+\sigma(x)$. Thus, the regret of SCRiBLe with lifting algorithm 
        \begin{eqnarray*}
            \duoE[\sum_{t=1}^{T}f_{t}(y_{t})-\sum_{t=1}^{T}f_{t}(x^{*})]&=
            &\duoE[\sum_{t=1}^{T}[\theta_{t}^\top y_{t}+\sigma_{t}(y_{t})]-\sum_{t=1}^{T}[\theta_{t}^\top x^{*}+\sigma_{t}(x^{*})]]\\ 
            &             = & \duoE[\sum_{t=1}^{T}\theta_{t}^\top y_{t}-\sum_{t=1}^{T}\theta_{t}^\top x^{*}]+\duoE[\sum_{t=1}^{T}\sigma_{t}(y_{t})-\sum_{t=1}^{T}\sigma_{t}(x^{*})].
        \end{eqnarray*}    
    Firstly, we bound the front of the above equation,
        \begin{eqnarray*}
            \duoE[\sum_{t=1}^{T}\theta_{t}^\top y_{t}-\sum_{t=1}^{T}\theta_{t}^\top x^{*}]&=
            &\sum_{t=1}^{T}\duoE[\theta_{t}^\top y_{t}]-\sum_{t=1}^{T}\duoE[\theta_{t}^\top x_{t}]+\sum_{t=1}^{T}\duoE[\theta_{t}^\top x_{t}]-\sum_{t=1}^{T}\duoE[\theta_{t}^\top x^{*}].
        \end{eqnarray*}
    From the Law of total expectation, we know
        \begin{eqnarray*}
            \sum_{t=1}^{T}\duoE[\theta_{t}^\top y_{t}]-\sum_{t=1}^{T}\duoE[\theta_{t}^\top x_{t}]&=&\sum_{t=1}^{T}\duoE[\theta_{t}^\top(y_{t}-x_{t})]\\  
            &=&\sum_{t=1}^{T}\duoE[\duoE_{t}[\theta_{t}^\top(y_{t}-x_{t})]]\\
            &=&\sum_{t=1}^{T}\duoE[\theta_{t}^\top\duoE_{t}[(\mathbf{A}_{t}\mu_{t})]]\\
            &=&\sum_{t=1}^{T}\duoE[\theta_{t}^\top\mathbf{0}]\\
            &=&\mathbf{0}.
         \end{eqnarray*}
    Thus,
        \begin{equation}
            \duoE[\sum_{t=1}^{T}\theta_{t}^\top y_{t}-\sum_{t=1}^{T}\theta_{t}^\top x^{*}]            =\duoE[\sum_{t=1}^{T}\theta_{t}^\top x_{t}-\sum_{t=1}^{T}\theta_{t}^\top x^{*}].
        \end{equation}
    From Lemma~\ref{lemm:bound-regret}, we have
        \begin{eqnarray*}
            \duoE[\sum_{t=1}^{T}\theta_{t}^\top x_{t}-\sum_{t=1}^{T}\theta_{t}^\top x^{*}]
                &\leq& \duoE[4\eta d^{2}T+\frac{\nu \ln (\frac{1}{\delta})}{\eta}+ dT(\nu +2\sqrt{\nu })(\frac{1-\delta}{\delta})\epsilon +\delta DGT]\\
                &\leq& 4\eta d^{2}T+\frac{\nu \ln (\frac{1}{\delta})}{\eta}+ dT(\nu +2\sqrt{\nu })(\frac{1-\delta}{\delta})\epsilon +\delta DGT.
        \end{eqnarray*}
     
Since $\sigma_{t}$ is chosen after knowing the player's action, it can cause as large a perturbation as possible. We using $\| \sigma_{t}(x) \|\leq \epsilon$ to bound $\sum_{t=1}^{T}\duoE[\sigma_{t}(y_{t})-\sum_{t=1}^{T}\sigma_{t}(x^{*})]\leq 2T\epsilon$ and
combination of Lemma~\ref{lemm:bound-regret}, we get
    \begin{eqnarray*}
    Regret&=&\duoE[\sum_{t=1}^{T}f_{t}(y_{t})-\sum_{t=1}^{T}f_{t}(x^{*})]\\
        &\leq& 4d\sqrt{\nu T\ln \frac{1}{\delta}}
    + 2dT(\nu +2\sqrt{\nu })(\frac{1-\delta}{\delta})\epsilon + \delta GDT + 2T\epsilon,
    \end{eqnarray*}
where $\eta=\frac{\sqrt{\nu\ln \frac{1}{\delta}}}{2d\sqrt{T}}$.
\end{proof}

\subsection{Proof of Theorem~\ref{Theo:high-probability-regret}}

To establish the high-probability regret bound, we first introduce the necessary Lemma 9.

\begin{lemma}[Theorem 2.2. in~\cite{lee2020bias}]
\label{lemm:martingale}
    Let $X_{1},...,X_{T}$ be a martingale difference sequence with respect to a filtration $F_{1}\subseteq...\subseteq F_{T}$ such that $\duoE[X_{t}\mid F_{t}]=0$. Suppose $B_{t}\in[1, b]$ for a fixed constant $b$ is $F_{t}$-measurable and such that $X_{t}\leq B_{t}$ holds almost surely. Then with probability at least $1-\gamma$
we have
\begin{equation}
    \sum_{t=1}^{T}X_{t}\leq C(\sqrt{8V\ln(C/\gamma)}+2B^{*}\ln(C/\gamma)),
\end{equation}
where $V=\max \{1; \sum_{t=1}^{T}\duoE[X_{t}^{2}\mid F_{t}]\}, B^{*}=\max_{t\in[T]}B_{t}$, and $C=\lceil \log b \rceil\lceil \log(b^{2}T) \rceil$.
\end{lemma}
The analysis in~\cite{lee2020bias} employs Lemma~\ref{lemm:martingale} to derive a
high-probability bound for $\sum_{t=1}^{T}[y_{t}^{\top}\theta_{t}-x_{t}^{\top}g_{t}+u^{\top}(g_{t}-\theta_{t})]$.
In contrast, our approach defines $X_{t}=\theta_{t}^{\top}y_{t}-\theta_{t}^{\top}x_{t}$ and derives the high-probability bound for $\sum_{t=1}^{T}(\theta_{t}^{\top}y_{t}-\theta_{t}^{\top}x_{t})$. This distinction in the application of Lemma~\ref{lemm:martingale} enables us to derive a tighter high-probability upper bound for bandit linear optimization.

With the support of Lemma~\ref{lemm:bound-regret} and Lemma~\ref{lemm:martingale}, we are ready to prove Theorem~\ref{Theo:high-probability-regret}.

\begin{proof}
    Let $X_{t}=\theta_{t}^{\top}y_{t}-\theta_{t}^{\top}x_{t}$, then
    $\duoE_{t}[X_{t}]=\duoE_{t}[\theta_{t}^{\top}y_{t}-\theta_{t}^{\top}x_{t}]=0$,
    $X_{t}=\theta_{t}^{\top}y_{t}-\theta_{t}^{\top}x_{t}\leq \|\theta_{t}\|\| y_{t}-x_{t}\|\leq GD$ and
    \begin{eqnarray*}
        \duoE_{t}[X_{t}^{2}]&=&\duoE_{t}[(\theta_{t}^{\top}y_{t}-\theta_{t}^{\top}x_{t})^{2}]\\
        &=&\duoE_{t}[(\theta_{t}^{\top}y_{t})^{2}+(\theta_{t}^{\top}x_{t})^{2}-2\theta_{t}^{\top}y_{t}\theta_{t}^{\top}x_{t}]\\
        &=&\duoE_{t}[(\theta_{t}^{\top}y_{t})^{2}]+\duoE_{t}[(\theta_{t}^{\top}x_{t})^{2}]-\duoE_{t}[2\theta_{t}^{\top}y_{t}\theta_{t}^{\top}x_{t}]\\
        &=&\duoE_{t}[(\theta_{t}^{\top}y_{t})^{2}]-\theta_{t}^{\top}x_{t}\theta_{t}^{\top}x_{t}\\
        &\leq&(1+\epsilon)^{2}.
    \end{eqnarray*}
    Then,
    \begin{eqnarray*}
            \sum_{t=1}^{T}f_{t}(y_{t})-\sum_{t=1}^{T}f_{t}(x^{*})&=
            &\sum_{t=1}^{T}[\theta_{t}^\top y_{t}+\sigma_{t}(y_{t})]-\sum_{t=1}^{T}[\theta_{t}^\top x^{*}+\sigma_{t}(x^{*})]\\ 
            &             \leq & \sum_{t=1}^{T}\theta_{t}^\top y_{t}-\sum_{t=1}^{T}\theta_{t}^\top x^{*}+\sum_{t=1}^{T}\sigma_{t}(y_{t})-\sum_{t=1}^{T}\sigma_{t}(x^{*})\\ 
            &             \leq & \sum_{t=1}^{T}\theta_{t}^\top y_{t}-\sum_{t=1}^{T}\theta_{t}^\top x^{*}+2T\epsilon\\
            &=&\sum_{t=1}^{T}\theta_{t}^\top y_{t}-\sum_{t=1}^{T}\theta_{t}^\top x_{t}+\sum_{t=1}^{T}\theta_{t}^\top x_{t}-\sum_{t=1}^{T}\theta_{t}^\top x^{*}+2T\epsilon.
    \end{eqnarray*}    
    From Lemma~\ref{lemm:bound-regret}, we know
    \begin{equation}
        \sum_{t=1}^{T}\theta_{t}^\top x_{t}-\sum_{t=1}^{T}\theta_{t}^\top x^{*}
        \leq 4\eta d^{2}T+\frac{\nu \ln (\frac{1}{\delta})}{\eta}+ 2dT(\nu +2\sqrt{\nu })(\frac{1-\delta}{\delta})\epsilon +\delta DGT,
    \end{equation}
    where $\eta=\frac{\sqrt{\nu \ln T }}{2d\sqrt{T}}$. 
    Then by Lemma~\ref{lemm:martingale},  
    \begin{equation}
        \duoP(\sum_{t=1}^{T}(\theta_{t}^{\top}y_{t}-\theta_{t}^{\top}x_{t})
        \leq C(\sqrt{8V \ln(C/\gamma)}+2B^{*}\ln(C/\gamma)))\geq 1-\gamma,
    \end{equation}    
where $V=(1+\epsilon)^{2}T, B^{*}=b=GD$, and $C=\lceil \ln  GD \rceil\lceil \ln ((GD)^{2}T) \rceil$.
    Combine everything to conclude the proof.
\end{proof}

\subsection{Application to black-box optimization}
From online to offline transformation, the result of this paper can also apply to black-box optimization for $\epsilon$-approximately linear functions.
This problem is important in that
previous theoretical analyses for black-box optimization can only deal with linear/convex/smooth objectives in the adversarial environments (via bandit convex optimization). So, it is quite meaningful to clarify the possibility of the black-box optimization problems without such restrictions.  In fact, our objective is not linear, nor smooth, even with a simple assumption.
 
Let $\hat{x}$ be the output of Algorithm 1.
Then, it follows from Theorem 1 that $f(\hat{x})-\min_{x\in \mathcal{K}}f(x)\leq \frac{4d\sqrt{\nu \ln \frac{1}{\delta}}}{\sqrt{T}} +\delta GD + d(\nu +2\sqrt{\nu })(\frac{1-\delta}{\delta})\epsilon + 2\epsilon$.
Additionally, we provide a lower bound $2\epsilon$ for this problem (see Lemma~\ref{lemm:blackbox-lower}). We can see that the difference between the lower bound and the upper bound is $\delta GD + d(\nu +2\sqrt{\nu })(\frac{1-\delta}{\delta})\epsilon$ as 
$T$ approaches infinity. 
This suggests the potential existence of "easier" settings between the adversarial environment and the standard stochastic environment, where better algorithms might be found. It also motivates us to explore these settings further.

\section{Lower bound}
In this section, 
we show a lower bound of the regret.
To do so, we consider a black-box optimization problem for 
the set $\Fset$ of $\epsilon$-approximately linear functions $f:\Kset \to \Real$.
In the problem, we are given access to the oracle $O_f$ for some $f\in \Fset$, 
which returns the value $f(x)$ given an input $x\in \Kset$. 
The goal is to find a point $\hat{x} \in \Kset$ such that 
$f(\hat{x}) -\min_{x\in \Kset} f(x)$ is small enough.
Then, the following statement holds.
 
\begin{lemma}
\label{lemm:blackbox-lower}
For any algorithm $\mathcal{A}$ for the black-box optimization problem for $\Fset$,
there exists an $\epsilon$-approximately linear function $f \in \Fset$ such that 
the output $\hat{x}$ of $\mathcal{A}$ satisfies 
\begin{equation}
f(\hat{x})-\min_{x\in\mathcal{K}}f(x) \geq 2\epsilon.
\end{equation}
\end{lemma}

\begin{proof}   
    Firstly, suppose that 
    the algorithm $\mathcal{A}$ is deterministic. 
    At iteration $t=1,...,T$, for any feedback $y_{1},...,y_ {t-1}\in\Real$,  
    $\mathcal{A}$ should choose the next query point $x_{t}$ based on the data observed so far. 
    That is, 
    \begin{equation}
        x_{t}=\mathcal{A}((x_{1}, y_{1}),...,(x_{t-1}, y_{t-1})).
    \end{equation}
    Assume that the final output $\hat{x}$ is returned after $T$ queries to the oracle $O_f$. 
    In particular, we fix the $T$ feedbacks $y_1=y_2=\dots=y_T=\epsilon$.
    Let $z\in\mathcal{K}$ be such that $z\notin\{x_{1},...,x_{T},\hat{x}\}$. 
    Then we define a function $f: \mathcal{K}\to\Real$ is as
    \begin{equation}
    \label{lower-bound-equation}
         f(x)= \begin{cases}
            \epsilon,\quad &x\not= z, \\
            -\epsilon,\quad &x=z.
        \end{cases}  
    \end{equation} 
    The function $f$ is indeed an $\epsilon$-approximately linear function, 
    as $f(x)=0^\top x + \sigma(x)$, where $\sigma(x)=\epsilon$ for $x \neq z$ and 
    $\sigma(x)=-\epsilon$ for $x =z$. 
    Further, we have 
    
    \begin{equation}
        f(\hat{x})-\min_{x\in\mathcal{K}} f(x) \geq 2\epsilon. 
    \end{equation}
    
    Secondly, if algorithm $\mathcal{A}$ is randomized. 
    It means each $x_{t}$ is chosen randomly.  
    We assume the same feedbacks $y_1=y_2=\dots=y_T=\epsilon$.
    Let $X=\{x_{1},...,x_{T},\hat{x}\}$.
    Then, there exists a point $z\in\mathcal{K}$ such that $P_{X}(z\in X)=0$, since 
    $\duoE_{z'}[P_{X}(z'\in X|z')]=P_{z',X}(z'\in X)=\duoE_{X}[P_{z'}(z'\in X|X)] =0$, 
    where the expectation on $z'$ is defined w.r.t.\ the uniform distribution over $\Kset$.
    For the objective function $f$ defined in (\ref{lower-bound-equation}), we have
    $f(\hat{x})-\min_{x\in\mathcal{K}} f(x) \geq 2\epsilon$ while $f$ is $\epsilon$-approximately linear.
\end{proof}

\begin{theorem}
For any horizon $T\geq 1$ and any player, there exists an adversary such that 
the regret is at least $2\epsilon T$.
\end{theorem}
\begin{proof}
We prove the statement by contradiction. 
Suppose that there exists a player whose regret is 
less than $2\epsilon T$. 
Then we can construct an algorithm for the black-box optimization problem from it by 
feeding the online algorithm with $T$ feedbacks of the black-box optimization problem and 
by setting $\hat{x}=\min_{t \in [T]}f(x_t)$. Then, 
\[
f(\hat{x}) - \min_{x \in \Kset}f(x) \leq \frac{\sum_{t=1}^T f(x_t) -\sum_{t=1}^T\min_{x\in \Kset}f(x) }{T} < 2\epsilon, 
\]
which contradicts Lemma~\ref{lemm:blackbox-lower}.
\end{proof}

This lower bound indicates that $\Omega(\epsilon T)$ regret is inevitable 
for the bandit optimization problem for $\epsilon$-approximately linear functions. 
We conjecture that the lower bound can be tightened to $\Omega(d\epsilon T)$, but we leave it as an open problem.

\section{Experiments}
\begin{figure}[h]
\includegraphics[width=0.8\textwidth]{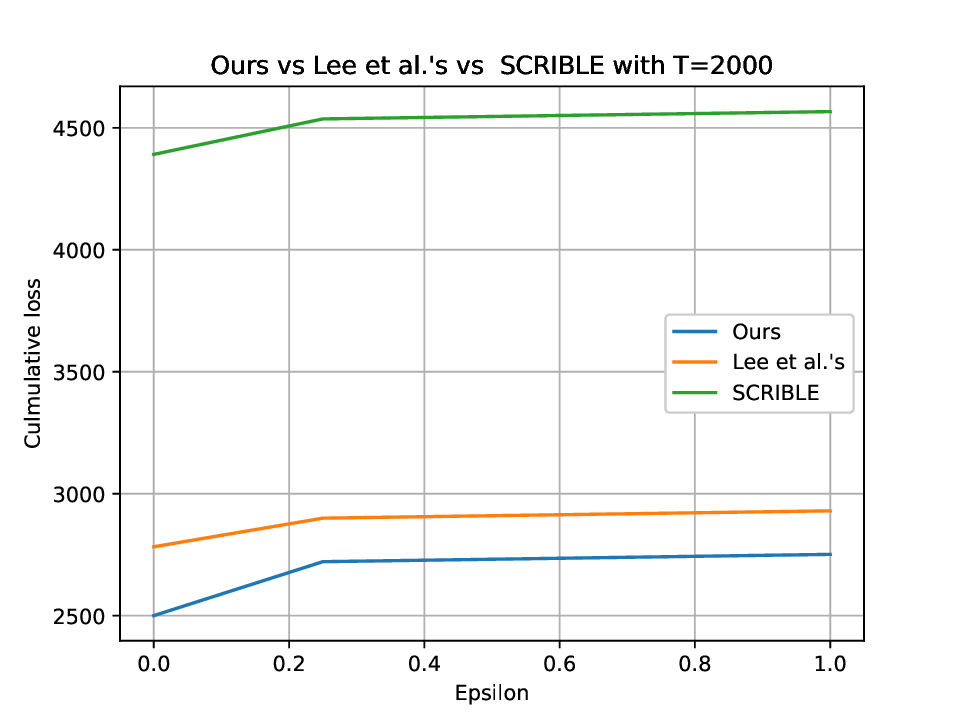}
\caption{Average cumulative loss of algorithms for artificial data sets. The blue line corresponds to the results of our algorithm, the yellow line represents the algorithm in~\cite{lee2020bias}, and the green line corresponds to the SCRiBLe algorithm~\cite{abernethy2008competing}.} \label{fig01}
\end{figure}
We conduct comparative experiments across a range of 
$\epsilon$ values: 0, 0.25, 0.5, 0.75, and 1, evaluating Algorithm 1, SCRiBLe~\cite{abernethy2008competing}, and SCRiBLe with lifting and increasing learning rates~\cite{lee2020bias}. 
We generate the artificial data as follows.
The input set is constructed such that each element $x \in \mathbb{R}^d$ satisfies $\|x\|_2 < D$. 
The loss vectors $\theta_{1},...,\theta_{T}$ are 
$d$-dimensional vectors randomly generated before the experiment, each with a Euclidean norm no greater than $G$.
The $\nu$-normal barrier is defined as 
$\mathcal{R}(x,b)=400 \cdot \left( -\log\left(1 - \frac{\|x\|^2}{b^2 D^2} \right) - 2\nu \log b \right)$ and $\eta$ is defined as $\frac{20 \sqrt{ \log \frac{1}{\delta}}}{4 d \sqrt{T}}$.
In our experiments, we set $D=5$, $G=1$, $d=5$, $\nu=1$, respectively.
The perturbation is defined as 
$\sigma(y_{t})=\epsilon\sin{((y_{t}^{\top}l)\pi)}$, where $l$ is a random vector in $\Real^{d}$ (Note that this is not the worst-case perturbation, which may explain why the results do not increase significantly as $\epsilon$ grows). 
Each algorithm is run with a fixed time horizon of $T=2000$. For each value of $\epsilon$, we repeat the experiment 10 times for each algorithm and compute the average cumulative loss.
The x-axis represents the value of 
$\epsilon$, and the y-axis shows the average cumulative loss.
As shown in the figure above, our method generally outperforms SCRiBLe. Compared to Lee et al.’s method, our approach achieves slightly better performance. This indicates that, when facing oblivious adversaries, not using increasing learning rates (i.e., dynamically adjusting the value of 
$\eta$) has little to no impact on the results(see
Fig.~\ref{fig01}).

\section{Conclusion}

In this work, we study bandit optimization with non-convex, non-smooth losses, where each loss consists of a linear term plus a small adversarial perturbation revealed after the action. We propose a new high-probability regret analysis and establish matching upper and lower bounds. Future work includes extending our results to broader loss classes, such as convex losses with perturbations.

\begin{credits}
\subsubsection{\ackname}
This work was supported by WISE program (MEXT) at Kyushu
University,  JSPS KEKENHI Grant Numbers JP23K24905,
JP23K28038, and JP24H00685, respectively.

\subsubsection{\discintname}
The authors have no competing interests to declare that are
relevant to the content of this article.
\end{credits}

%
%
%

\bibliography{main, hatano}
\bibliographystyle{splncs04}
 
\appendix
\section{Appendix}
Before prove the Lemma~\ref{lemm:base-normal-barrier}, we introduce the follow Lemma first.
\begin{lemma}
[\cite{nemirovski2004interior}]
    \label{lemm:Dikins ellipsoids}
        For $x\in int(\mathcal{K})$ and $h\in \Real^d$ let
     $p_x(h) = inf\{r \geq0|x \pm r^{-1}h\in \mathcal{K}\}$,
     One has
     $p_x(h)<=|h|_x \leq
    (\nu+2\sqrt{\nu} )p_x(h)$.
\end{lemma}

\subsection{Proof of Lemma 4}
\begin{proof}
    From Lemma~\ref{lemm:Dikins ellipsoids}, for $\forall x,y\in\mathcal{K}_{\delta}$, we have 
    \begin{equation}
        \| x-y \|_{x} \leq (\nu+2\sqrt{\nu})p_x(x-y).
    \end{equation}
    For $p_x(x-y) = inf\{r \geq0 | x \pm r^{-1}(x-y)\in \mathcal{K}\}$, it suffices to show $\| x \pm r^{-1}(x-y) \| \leq D$.
    \begin{equation}
        \| x \pm r^{-1}(x-y) \| \leq \| x \| + \| r^{-1}(x-y) \|\leq (1-\delta)D+r^{-1}2(1-\delta)D
    \end{equation}
    Then, $r = 2\frac{(1-\delta)}{\delta}$ and $\| x-y \|_{x} \leq 2(\nu+2\sqrt{\nu})\frac{(1-\delta)}{\delta}$
\end{proof}

\subsection{Proof of Lemma 5}
\begin{proof}
    Recall that $x'_{t+1}=\arg\min_{x'\in\mathcal{K'}}\phi_{t}(x')$, where $\phi_{t}(x')=\eta\sum_{\tau=1}^{t}g_{\tau}^{\top}x'+\mathcal{R}(x')$.
    Let $h_{t}(x)=\phi_{t}((x, 1))=\phi_{t}(x')$, then $\min h_{t}(x)=\min\phi_{t}(x')$. Noticing that $h_{t}$ is a convex function on $\Real^{d}$ and still holds the barrier property(approaches infinity along any sequence
    of points approaching the boundary of $\mathcal{K}$).
    By properties of convex functions, we can get
    $\nabla h_{t-1}(x_{t})=0$ and for the first $d$ coordinates $\nabla \phi_{t-1}(x'_{t})=0$.
    
    Consider any point in $z\in W_{\frac{1}{2}}(x'_{t})$. It can be written as
    $z=x'_{t}+\alpha u$ for some vector $u$ such that $\| u \|_{x'_{t}}=1$ and $\alpha\in(-\frac{1}{2},\frac{1}{2})$. Noticing the $d+1$ coordinate of $u$ is 0. 
    Expanding,
    \begin{eqnarray*}
         \phi_{t}(z)&=&\phi_{t}(x'_{t}+\alpha u)\\
         &=&\phi_{t}(x'_{t})+\alpha\nabla\phi_{t}(x'_{t})^{\top}u+\alpha^{2}\frac{1}{2}u^{\top}\nabla^{2}\phi_{t}(\xi)u\\
         &=&\phi_{t}(x'_{t})+\alpha(\nabla\phi_{t-1}(x'_{t})+\eta g_{t})^{\top}u+\alpha^{2}\frac{1}{2}u^{\top}\nabla^{2}\phi_{t}(\xi)u\\
         &=&\phi_{t}(x'_{t})+\alpha\eta g_{t}^{\top}u+\alpha^{2}\frac{1}{2}u^{\top}\nabla^{2}\phi_{t}(\xi)u,
    \end{eqnarray*}
    for some $\xi$ on the path between $x'_{t}$ and $x'_{t}+\alpha u$ and the last equality holds because $\nabla\phi_{t-1}(x'_{t})^{\top}u=0$.
    Setting the derivative with respect to $\alpha$ to zero, we obtain 
    \begin{equation}
        \| \alpha^{*}\|=\frac{\eta \| g_{t}^{\top}u\|}{u^{\top}\nabla^{2}\phi_{t}(\xi)u}=\frac{\eta \| g_{t}^{\top}u\|}{u^{\top}\nabla^{2}\mathcal{R}(\xi)u}
    \end{equation}
    The fact that $\xi$ is on the line $x'_{t}$ to $x'_{t}+\alpha u$ implies that $\| \xi-x'_{t}\|_{x'_{t}}\leq \| \alpha u\|_{x'_{t}}\leq \frac{1}{2}$. Hence, by Lemma~\ref{lemm:normal-property}
    \begin{equation}
        \nabla^{2}\mathcal{R}(\xi)\succeq (1-\| \xi-x'_{t}\|_{x'_{t}})^{2}\nabla^{2}\mathcal{R}(x'_{t})\succ \frac{1}{4}\nabla^{2}\mathcal{R}(x'_{t}).
    \end{equation}
    Thus $u^{\top}\nabla^{2}\mathcal{R}(\xi)u>\frac{1}{4}\| u \|_{x'_{t}}=\frac{1}{4}$, and $\alpha^{*}<4\eta\| g_{t}^{\top}u\|$.
    Using assumption $\max_{x\in\mathcal{K}} \| f_{t}(x) \| \leq 1$,
    \begin{equation}
        g_{t}^{\top}u\leq  \| g_{t}\|^{*}_{x'_{t}}\| u \|_{x'_{t}}\leq\| df_{t}(y_{t})\mathbf{A}_{t}^{-1}\mu_{t}\|^{*}_{x_{t}'}\leq\sqrt{d^{2}\mu_{t}^{\top}\mathbf{A}_{t}^{-\top}(\nabla^{2}\mathcal{R}(x_{t}'))^{-1}\mathbf{A}_{t}^{-1}\mu_{t}}\leq d, \label{eq:12}
    \end{equation}
    we conclude that $\| g_{t}^{\top}u\|\leq d$, and $\| \alpha^{*}\|<4d\eta<\frac{1}{2}$ by our choice of $\eta$ and $T$.  We conclude that the local optimum $\arg\min z\in W_{\frac{1}{2}(x'_{t})}\phi_{t}(z)$ is strictly inside $W_{4d\eta}(x'_{t})$,
    and since $\phi_{t}$ is convex, the global optimum is
    \begin{equation}
        x_{t+1}=\arg\min_{z\in\mathcal{K'}}\phi_{t}(z)\in W_{4d\eta}(x'_{t}).
    \end{equation}
\end{proof}
\subsection{Proof of Lemma 8}
\begin{proof}
    Recall for any $\delta \in (0,1)$, 
    $\mathcal{K}_{\delta}=\{x|\frac{1}{1-\delta} x \in \mathcal{K}\}$ and 
    $\mathcal{K'_{\delta}}=\{(x,1):x\in\mathcal{K_\delta}\}$.
    Let $x_{\delta}^{*}=\Pi_{\mathcal{K}_{\delta}}x^{*}$, by properties of projections, then
    \begin{eqnarray}
        \| x^{*}-x_{\delta}^{*}\|=\min_{a\in\mathcal{K_{\delta}}}\| x^{*}-a\|.
    \end{eqnarray}
    Since $(1-\delta)x^{*}\in\mathcal{K_{\delta}}$, then
    \begin{equation}
       \min_{a\in\mathcal{K_{\delta}}}\| x^{*}-a\|
        \leq \| x^{*}-(1-\delta)x^{*}\|
        \leq \delta D.
    \end{equation}
    \begin{equation}
        \| x_{\delta}^{*}-x^{*}\|\leq\delta D.
    \end{equation}

    By Cauchy–Schwarz inequality and the fact that $\| \theta_{t} \|\leq G$ and $\| x_{\delta}^{*}-x^{*}\|\leq\delta D$,
    \begin{equation}
        \sum_{t=1}^{T}\theta_{t}^\top x_{\delta}^{*}-\sum_{t=1}^{T}\theta_{t}^\top x^{*}\leq\delta DGT.
    \end{equation}
    Then,
     \begin{eqnarray*}
       \sum_{t=1}^{T}\theta_{t}^\top x_{t}-\sum_{t=1}^{T}\theta_{t}^\top x^{*}
        &=&\sum_{t=1}^{T}\theta_{t}^\top x_{t}-\sum_{t=1}^{T}\theta_{t}^\top x_{\delta}^{*}+\sum_{t=1}^{T}\theta_{t}^\top x_{\delta}^{*}-\sum_{t=1}^{T}\theta_{t}^\top x^{*}\\
        &\leq& \sum_{t=1}^{T}\theta_{t}^\top x_{t}-\sum_{t=1}^{T}\theta_{t}^\top x_{\delta}^{*}+\delta DGT.
    \end{eqnarray*}

    Let $\theta_{t}'=(\theta_{t}, z)$, where $z$ is the $(d+1)$th coordinate of $d\duoE_{t}[(\theta_{t}, 0)^{\top}(x_{t}'+\mathbf{A}_{t}\mu_{t})\mathbf{A}_{t}^{-1}\mu_{t}]$.
    From Lemma~\ref{lemm:gradient-estimate}, we know $d\duoE_{t}[(\theta_{t}, 0)^{\top}(x_{t}'+\mathbf{A}_{t}\mu_{t})\mathbf{A}_{t}^{-1}\mu_{t}]=\theta_{t}'$. 
    Since $g_{t}=df(y_{t})\mathbf{A}_{t}^{-1}\mu_{t}=d(\theta_{t},0)^{\top}(x_{t}'+\mathbf{A}_{t}\mu_{t})\mathbf{A}_{t}^{-1}\mu_{t}+d\sigma_{t}(y_{t})\mathbf{A}_{t}^{-1}\mu_{t}$, and letting $M_{t}=\duoE_{t}[d\sigma_{t}(y_{t})\mathbf{A}_{t}^{-1}\mu_{t}]$, it follows that $\theta_{t}'=\duoE_{t}[g_{t}]-M_{t}$. Consequently,
        \begin{eqnarray*}
             \sum_{t=1}^{T}\theta_{t}^\top x_{t}-\sum_{t=1}^{T}\theta_{t}^\top x^{*}_{\delta}
            &=&\sum_{t=1}^{T}(\theta_{t}'^\top x_{t}'-z)-\sum_{t=1}^{T}(\theta_{t}'^\top x^{*'}_{\delta}-z)\\
            &=&\sum_{t=1}^{T}(\duoE_{t}[g_{t}]-M_{t})^{\top}x_{t}'-\sum_{t=1}^{T}(\duoE_{t}[g_{t}]-M_{t})^{\top}x_{\delta}^{*'}\\
            &=&\sum_{t=1}^{T}\duoE_{t}[g_{t}]^{\top}(x_{t}'-x_{\delta}^{*'})+\sum_{t=1}^{T}M_{t}^{\top}(x_{\delta}^{*'}-x_{t}').
        \end{eqnarray*}

    We bound $\sum_{t=1}^{T}M_{t}^{\top}(x_{\delta}^{*'}-x_{t}')$ firstly.
    By Cauchy–Schwarz inequality,
    \begin{eqnarray*}
            \sum_{t=1}^{T}M_{t}^{\top}(x_{\delta}^{*'}-x_{t}')
            &\leq&\sum_{t=1}^{T}\| M_{t} \|_{x_{t}'}^{*} \| x_{\delta}^{*'}-x_{t}' \|_{x_{t}'} 
            .
    \end{eqnarray*}
By Jensen's inequality,
    \begin{eqnarray}
            \| M_{t} \|_{x_{t}'}^{*} &=&
            \sqrt{M_{t}^{\top}\nabla^{2}(\mathcal{R}(x_{t}'))^{-1}M_{t}}\label{eq:boundM1}\\
            &=&\sqrt{\duoE_{t}[d\sigma_{t}(y_{t})\mathbf{A}_{t}^{-1}\mu_{t}]^{\top}\nabla^{2}(\mathcal{R}(x_{t}'))^{-1}\duoE_{t}[d\sigma_{t}(y_{t})\mathbf{A}_{t}^{-1}\mu_{t}]}\\
            &=&\sqrt{d^{2} \duoE_{t}[\sigma_{t}(y_{t})\mu_{t}]^{\top}\mathbf{A}_{t}^{-1}\mathbf{A}_{t}^{2}\mathbf{A}_{t}^{-1}\duoE_{t}[\sigma_{t}(y_{t})\mu_{t}]}
    \\ &=& \sqrt{d^{2} \duoE_{t}[\sigma_{t}(y_{t})\mu_{t}]^{\top}\duoE_{t}[\sigma_{t}(y_{t})\mu_{t}]}\\
    &=& \sqrt{d^{2} \| \duoE_{t}[\sigma_{t}(y_{t})\mu_{t}]\|^2}\\
            &\leq&\sqrt{d^{2}\duoE_{t}[\|\sigma_{t}^{2}(y_{t})\mu_{t}^{\top}\mu_{t}\|]}\\
            &\leq&\sqrt{d^{2}\epsilon^{2}}\\
            &=&d\epsilon. \label{eq:boundM2}
    \end{eqnarray}
    
    Then we bound $\| x_{\delta}^{*'}-x_{t}' \|_{x_{t}'}$. By Lemma~\ref{lemm:control-h},
    \begin{eqnarray*}
            \| x_{\delta}^{*'}-x_{t}' \|_{x_{t}'} &\leq& 2(\nu+2\sqrt{\nu})\frac{(1-\delta)}{\delta}.
    \end{eqnarray*}

So $\sum_{t=1}^{T}M_{t}^{\top}(x_{\delta}^{*'}-x_{t}')\leq Td\epsilon 2(\nu+2\sqrt{\nu})\frac{(1-\delta)}{\delta}$.
Then bound $\sum_{t=1}^{T}\duoE_{t}[g_{t}]^T(x_{t}'-x_{\delta}^{*'})$

By Lemma~\ref{lemm:FTRL},
    \begin{eqnarray*}
            \sum_{t=1}^{T}\duoE_{t}[g_{t}]^T(x_{t}'-x_{\delta}^{*'})
            &=&\duoE_{t}\{\sum_{t=1}^{T}g_{t}^{\top}(x_{t}'-x_{\delta}^{*'})\}\\
            &\leq&\duoE_{t}\{\sum_{t=1}^{T}[g_{t}^{\top}x_{t}'-g_{t}^{\top}x_{t+1}']+\frac{1}{\eta}[\mathcal{R}(x_{\delta}^{*'})-\mathcal{R}(x_{1}')]\}\\
            &\leq&\duoE_{t}\{\sum_{t=1}^{T}[\| g_{t} \|^{*}_{x_{t}'} \| x_{t}'-x_{t+1}' \|_{x_{t}'}]\}+\frac{1}{\eta}(\mathcal{R}(x_{\delta}^{*'})-\mathcal{R}(x_{1}')).
    \end{eqnarray*}

Lemma~\ref{lemm:close-property} implies that $\| x'_{t}-x'_{t+1}\|_{x'_{t}}\leq 4d\eta$ is true by choice of $\eta$. Additionally, from Eq.~(\ref{eq:12}), we deduce that $\| g_{t}\|^{*}_{x_{t}'}\leq d$.
 Therefore,
 \begin{equation}
     \| g_{t} \|^{*}_{x_{t}'} \| x_{t}'-x_{t+1}' \|_{x_{t}'}\leq 4\eta d^{2},
 \end{equation}
 \begin{equation}
     \duoE_{t}\{\sum_{t=1}^{T}[\| g_{t} \|^{*}_{x_{t}'} \| x_{t}'-x_{t+1}' \|_{x_{t}'}]\}
            \leq 4\eta d^{2}T .
 \end{equation}

With Lemma~\ref{lemm:normal-log-property},
    \begin{equation}
        \frac{1}{\eta}(\mathcal{R}(x_{\delta}^{*'})-\mathcal{R}(x_{1}'))
                \leq \frac{\nu \log (\frac{1}{\delta})}{\eta}.
    \end{equation}
Combine everything, we get
    \begin{equation}
         \sum_{t=1}^{T}\theta_{t}^\top x_{t}-\sum_{t=1}^{T}\theta_{t}^\top x^{*}_{\delta}
                \leq 4\eta d^{2}T+\frac{\nu \log (\frac{1}{\delta})}{\eta}+ Td\epsilon2(\nu+2\sqrt{\nu})\frac{(1-\delta)}{\delta}+\delta DGT.
    \end{equation}
\end{proof}

\end{document}